\documentclass{article}

\usepackage{times}
\usepackage{graphicx} 
\usepackage{wrapfig}
\usepackage{subfigure} 

\usepackage{array}

\usepackage{enumitem}

\usepackage{natbib}

\usepackage{algorithm}
\usepackage{algorithmic}

\usepackage{amsmath}
\usepackage[T1]{fontenc}
\usepackage[utf8x]{inputenc}
\usepackage{amsfonts} 
\usepackage{amsthm} 

\usepackage[svgnames]{xcolor}

\usepackage{hyperref}

\usepackage[accepted]{icml2015}

\def\:#1{\mathbf{#1}}

\renewcommand{\Re}{\mathbb{R}}

\newcommand{\autoeqref}[1]{\hyperref[#1]{\equationautorefname~(\ref*{#1})}}
\newcommand{\wt}[1]{\widetilde{#1}}
\newcommand{\wh}[1]{\widehat{#1}}
\newcommand{\transp}{\mathsf{T}}
\DeclareMathOperator*{\argmin}{arg\,min}

\DeclareMathOperator*{\Ker}{Ker}

\newcommand{\normsmall}[1]{\Vert #1 \Vert}

\newcommand{\Gg}{\mathcal{G}}
\newcommand{\Hg}{\mathcal{H}}
\newcommand{\Ag}{\mathcal{A}}
\newcommand{\vertexset}{\mathcal{V}}
\newcommand{\edgeset}{\mathcal{E}}

\newcommand{\funcspace}{\mathcal{F}}
\newcommand{\dataset}{\mathcal{D}}
\newcommand{\trainingset}{\mathcal{S}}
\newcommand{\testset}{\mathcal{T}}

\newcommand{\vareps}{\varepsilon}
\newcommand{\bigotime}{\mathcal{O}}
\DeclareMathOperator*{\polylog}{polylog}

\newtheorem{theorem}{Theorem}

\newtheorem{definition}{Definition}

\usepackage[colorinlistoftodos, textwidth=40mm, shadow]{todonotes}

\icmltitlerunning{Large-scale semi-supervised learning with online spectral graph sparsification}

\begin{document} 

\twocolumn[
\icmltitle{Large-scale semi-supervised learning with online spectral graph sparsification}

\icmlauthor{Daniele Calandriello}{daniele.calandriello@inria.fr}
\icmlauthor{Alessandro Lazaric}{alessandro.lazaric@inria.fr}
\icmlauthor{Michal Valko}{michal.valko@inria.fr}
\icmladdress{Team SequeL Inria Lille -- Nord Europe, France}

\icmlkeywords{Online-Learning, Semi-Supervised Learning, Spectral Sparsification}

]

\section{Introduction}

In many classification and regression tasks, obtaining many good-quality 
labeled examples may be expensive.
When the number of labeled examples is very small, traditional 
supervised learning algorithms fail in learning accurate predictors. 
\emph{ Semi-supervised learning }\citep[SSL,][]{chapelle2006semi-supervised} deals with this problem by 
integrating the labeled examples with an additional set of unsupervised 
samples to make use of an underlying structure
(e.g., a manifold) and reduced the need for labeling.
In this paper we consider data whose similarity can be encoded in a \emph{graph}, 
and the similarity between nodes is much easier to obtain
than their label. 
Given the graph, SSL methods leverage the assumption that nodes which are 
similar according to the graph are more likely to be labeled similarly. 
Graph-based SSL propagates the labels from the labeled nodes
to the unlabeled ones.
%
For instance, the objective of \emph{harmonic function solution}
\citep[HFS,][]{zhu_semi-supervised_2003, belkin_regularization_2004} is to find a solution where each
node's value is the weighted average of its neighbors.
The HFS solution can be found solving a linear system involving the graph
Laplacian, but for dense graphs on $n$ nodes
this amounts to $\bigotime(n^3)$ time complexity and a $\bigotime(n^2)$ space complexity,
which is infeasible for large $n$.
In this paper, we consider a more realistic setting when \emph{space and 
computational budgets are limited}. In particular, we only allow
$\bigotime(n \polylog(n))$ space for 
storing the graph structure and an amortized computational cost of
$\bigotime(\polylog(n))$ for each of the edges in the original graph. Notice that 
these  constraints make it even impossible to store the full similarity 
matrix in memory. To this end we employ efficient
online spectral graph sparsification techniques \cite{kelner_spectral_2013}
to incrementally process the stream. This, coupled with specific solvers for symmetric diagonally dominant (SDD)
matrices \cite{koutis_solving_2011}, allow us to never store the whole
graph in memory  and to control the computational complexity as the number of nodes
grows. Using the approximation properties of spectral sparsifiers and
with results from algorithmic stability theory \cite{bousquet_stability_2002, cortes_stability_2008}
we will provide theoretical guarantees for the generalization error 
for this approximation.

%


\section{SSL with Spectral Sparsification}

\textbf{Notation.}
We denote with lowercase letter $a$ a scalar, with bold lowercase letter $\:a$ 
a vector and with uppercase letters $A$ a matrix.
We consider the general transductive setting, where we assume that there 
exists a dictionary of labeled nodes $\dataset = \{(x_i,y_i)\}_{i=1}^n$, where 
the nodes are organized over an undirected graph $\Gg = (\vertexset,\edgeset)$ 
with $n$ vertices $\vertexset = \{1,\ldots, n\}$ and $m$ edges, and the labels 
are $y_i\in\Re$.
Given graphs $\Gg, \Ag$ defined on the same vertex set, the graph $\Gg+\Ag$
is obtained by adding the weights on the edges of $\Ag$ to $\Gg$. In a similar
manner we define $\Gg + e$ for edge $e$.
For $i \in \vertexset$, we denote with $\chi_i$ the indicator vector, and with $\:b_e$
the vector $\chi_i - \chi_j$.
While the algorithm receives information on the features $x_i$ 
of all nodes, only a limited (random) subset $\trainingset$ of $l$ nodes is 
actually labeled. The objective of the learning algorithm is to minimize the 
error on the complementary unlabeled set 
$\testset = \dataset \backslash \trainingset$. More precisely, the objective 
is to learn a function $\:f : \vertexset \rightarrow \Re$ that minimizes the 
generalization error $R(\:f) = \tfrac{1}{u}\sum_{i=1}^{u} (\:f(x_i) - \:y(x_i))^2$, 
where $u = |\testset| = n-l$ is the number of unlabeled nodes.
We indicate with $\:f$ and $\:y \in \Re^n$ the vectors that contain
the function and the labels evaluated at the $n$ points $x_i$.

\textbf{Stable-HFS.}
HFS exploits the graph structure to learn functions that 
predict similar values $y$ for similar nodes. Given a weighted adjacency 
matrix $A_{\Gg}$, with edge weights $a_e$, and the degree matrix $D_{\Gg}$, the 
Laplacian is defined as $L_\Gg = D_\Gg - A_\Gg$. We  assume 
the graph is connected. In this case, $L_{\Gg}$ is semi-definite positive 
(SDP) with $\Ker(L_\Gg)~=~\:1$. Let $L_{\Gg}^+$ be the pseudoinverse of 
$L_{\Gg}$, and  $L_{\Gg}^{-1/2} = (L_{\Gg}^+)^{1/2}$.
The original HFS method~\cite{zhu_semi-supervised_2003}, when
we allow the label of already labeled nodes to change, can be formulated as the
Laplacian-regularized least-squares problem
\begin{align}
    \wh{\:f} &= \argmin_{\:f \in \Re^n}  \tfrac{1}{l} (\:f - \:y)^\transp I_{\trainingset} (\:f-\:y) + \gamma \:f^\transp L_\Gg \:f\nonumber\\
            &= (\gamma l L_\Gg + I_\trainingset)^+ (\:y ), \label{eq:def-solution}
\end{align}
where $I_{\trainingset}$ is the identity matrix with zeros corresponding to 
the nodes not in $\trainingset$, and $\gamma$ is a regularizer. 
While HFS achieves interesting empirical results, it is not easy to
provide theoretical guarantees for it  due to the singularity of the
Laplacian matrix. For this reason, we focus on 
the stable-HFS algorithm proposed by~\citet{belkin_regularization_2004} where 
an additional regularization term is introduced to restrict the space of 
admissible hypothesis, so that 
$\funcspace = \left\{\:f : \langle\:f,\:1\rangle = 0\right\}$.
This restriction can be enforced introducing an additional
$\mu$ regularization term, that can be computed in closed form as
\begin{align}
    \mu = (\gamma l L_\Gg + I_\trainingset)^+\:y/(\gamma l L_\Gg + I_\trainingset)^+ \:1 \label{eq:def-mu},
\end{align}
and subtracting $\mu\:1$ from the unconstrained solution.
It can be shown that this is  equivalent to projecting
the unregularized solution using the projection matrix
$P_\funcspace = L_\Gg L_{\Gg}^+$.
While stable-HFS is more suited for theoretical analysis,
its computational and space requirements remain
polynomial.
If the graph $\Gg$ has no particular
property, solving the linear system takes $\bigotime(n^2)$ space
and $\bigotime(n^3)$ time. To satisfy our resource constraints, we include spectral sparsification in stable-HFS.
Computing the solution on a sparse graph $\Hg$ that approximates $\Gg$
removes the polynomial complexity.

\begin{algorithm}[t!]
\begin{algorithmic}
\begin{small}
    \INPUT $\{x_i: i \in \dataset\},\{y_i : i \in \trainingset\}$, a stream
    of $m$ edges $\edgeset$
    \OUTPUT $\wh{\:f}, \Hg$
    \STATE Initialize $\Hg = \emptyset, \Ag = \emptyset, t = 1$
    \WHILE{$t \leq m$}
        \FOR{$|\Ag| \leq n \log^2(n)/\varepsilon^2$}
            \STATE Receive edge $e_t$ and add it to $\Ag$
            \STATE $t = t+1$
        \ENDFOR
        \STATE Compute a new graph $\Hg$ using Alg.~\ref{alg:kl_resparsify}
        on $\Hg + \Ag$
        \STATE Build Laplacian $L_\Hg$ and diag. matrix $\{ I_{\trainingset}(i,i) = 1 : i \in \trainingset \}$
        \STATE Compute HFS $\wt{\:f}$ using Eq.~\ref{eq:def-solution} and $L_\Hg$
        \STATE $\wt{\:f} = \wt{\:f} - \mu \:1$ where $\mu$ is computed using Eq.~\ref{eq:def-mu}
    \ENDWHILE
\end{small}
\end{algorithmic}
\caption{Sparse-HFS}\label{alg:sparse_ssl}
\end{algorithm}

\begin{algorithm}[h]
\begin{algorithmic}
\begin{small}
\INPUT $\Hg,\Ag$, the previous probabilities $\wt{p}_e$ for all 
    edges in $\Hg$ and the weights of the edges $a_e$.
    \OUTPUT $\Hg'$, a $1 \pm \vareps$ sparsifier of $\Gg' = \Gg + \Ag$ and new 
    prob. $\{\wt{p}'_e: e \in \Hg'\}$.
    \STATE $\alpha^2 = 1/(1-\vareps)^2, N = \alpha^2 n \log^2 (n) / \vareps^2$
    \STATE Obtain estimates $\{\wt{R}'_e: e \in \Hg+\Ag\}$  such that \\ \quad $1/\alpha \leq \wt{R}'_e/R'_e \leq \alpha  $
    with an SDD solver \cite{koutis_solving_2011}
    \STATE Compute prob. $\wt{p}'_e = (a_e \wt{R}'_e)/(\alpha(n-1))$  and $w_e = a_e/(N\wt{p}'_e)$
    \FOR{all edges $e \in \Hg$}
        \STATE $\wt{p}'_e \leftarrow \min\{\wt{p}_e,\wt{p}'_e\}$
    \ENDFOR
    \STATE Initialize $\Hg' = \emptyset$
    \FOR{all edges $e \in \Hg$}
    \STATE with probability $\wt{p}'_e/\wt{p}_e$ add edge $e$ to $\Hg'$ with weight $w_e$
    \ENDFOR
    \FOR{all edges $e \in \Ag$}
        \STATE /*The inner loop is run implicitly by sampling a binomial*/
        \FOR{$ i = 1 $ to $N$}
            \STATE with probability $\wt{p}'_e$ add edge $e$ to $\Hg'$ with weight $w_e$
        \ENDFOR
    \ENDFOR
    \end{small}
\end{algorithmic}
\caption{Kelner-Levin Sparsification Algorithm}\label{alg:kl_resparsify}
\end{algorithm}

\textbf{Sparse-HFS.} Spectral sparsifiers have been central in the 
development of efficient linear solvers \cite{koutis_solving_2011}. Since their 
introduction by \citet{spielman_spectral_2011}, they were extended
to insertion-only streams \cite{kelner_spectral_2013}.
\begin{definition}\label{def:eps-sparsifier}
    A $ 1 \pm \vareps$ spectral sparsifier of $\Gg$ is a graph $\Hg \subseteq \Gg$
    such that for all $\:x$
    \begin{align*}
        &(1-\vareps)\:x^\transp L_{\Gg} \:x \leq \:x^\transp L_\Hg \:x \leq (1 + \vareps) \:x^\transp L_{\Gg} \:x
    \end{align*}
\end{definition}
In this paper, we propose to spectral sparsify $\Gg$ to reduce the complexity of HFS.
A sparse graph $\Hg$ can be stored efficiently, but if the construction
of the sparsifier requires access to the whole $\Gg$ graph at every moment,
just storing the original graph in memory can be impossible. Moreover,
traditional linear solvers for an $n \times n$ matrix with $m$ nonzero entries
have a time complexity of $\bigotime(mn)$, which is already infeasible for $m = n$.
To meet our space and time requirements, we propose to build the sparsifier incrementally
using Alg.~\ref{alg:kl_resparsify}~\cite{kelner_spectral_2013}.
Sparse-HFS (Alg.~\ref{alg:sparse_ssl}) receives as input a previous sparsifier $\Hg$ and a
stream of edges insertions (i.e., from a disk or a network)
and stores them in memory until
a graph $\Ag$ with $\bigotime(n \polylog(n))$ edges has formed. At this point,
the sparsifier $\Hg$ gets updated, generating a new sparsifier that again
occupies only $\bigotime(n \polylog(n))$ space.
The key component in generating the sparsifier is random sampling according
to the effective resistances.
The effective resistance of an edge $e$ is defined as $R_e = \:b_e^\transp L_{\Hg}^+ \:b_e$.
Computing $R_e$ na\"{\i}vely requires again $\bigotime(n^2\polylog(n))$ time
and is not feasible in general.
Using linear solvers for SDD matrices \cite{koutis_solving_2011}
we can get $R_e$ for all edges in $\Hg$
in $\bigotime(n\polylog(n))$ time.
Using the same solver, recomputing the updated solution $\wt{\:f}$
for the updated sparsifier $\Hg$ takes the same time. Therefore, the whole update procedure
takes $\bigotime(n \polylog(n))$. Note that updating the
solution at each step is still possible, but it will not meet the
computational budget of a $\bigotime(\polylog(n))$ amortized cost.

\section{Theoretical Analysis}

By a good approximation of quadratic forms, spectral sparsifiers give 
many guarantees on eigenvalues, eigenvectors and solutions to linear 
systems.
Let $P_{\funcspace} = L_{\Gg} L_{\Gg}^+$ be the
projection matrix on the $n-1$ dimensional space
$\Ker(L_{\Gg})^\transp = \funcspace$.
We derive a bound on the generalization error for sparse-HFS and compare 
it to the original stable-HFS. 
We start with the definition of a stable algorithm.

\begin{definition}[Transduction $\beta$-stability]\label{def:beta-stability}
    Let $\mathcal{L}$ be a
    transductive learning algorithm and let $\:f$ denote the
    hypothesis returned by $\mathcal{L}$ for $\dataset = (\trainingset, \testset)$ and $\:f'$ the hypothesis
    returned for $\dataset = (\trainingset', \testset')$. $\mathcal{L}$ is uniformly
    $\beta$-stable with respect to the squared loss if there exists $\beta \geq 0$ such that for any two partitions
    $\dataset = (\trainingset, \testset)$ and $\dataset = (\trainingset', \testset')$ that differ in exactly one
    training (and thus test one) point and for all $x \in \dataset$,
    \begin{align*}
        |(\:f(x) - \:y(x))^2 - (\:f'(x) - \:y(x))^2| \leq \beta.
    \end{align*}
\end{definition}

The analysis of algorithmic stability \cite{bousquet_stability_2002}
has been extensively used
in statistic for concentration inequalities in the transductive setting
\cite{el-yaniv_stable_2006} and later for algorithmic guarantees
\cite{cortes_stability_2008}.
Define the empirical error as $\wh{R}(\:f) = \tfrac{1}{l} \sum_{i=1}^{l} (\:f(x_i) - \:y(x_i))^2$
and the generalization error as $R(\:f) = \tfrac{1}{u}\sum_{i=1}^{u} (\:f(x_i) - \:y(x_i))^2$.

\begin{theorem}
    \label{thm:sparse-ssl-generalization}
    Let $|\:f(x) - \:y(x)| \leq c$ and $|\:y(x)| \leq k$ for all $x \in
    \dataset, \:f \in \funcspace$. Let
    $\wt{\:f}$ be the hypothesis returned by sparse-HFS
    (Alg.~\ref{alg:sparse_ssl}) when trained on
    $\dataset = (\trainingset, \testset)$, and $\wh{\:f}$ the solution
    returned by stable-HFS.
    Then for any $\delta > 0$, with probability at least $1 - \delta$,
    \begin{align*}
        R(\wt{\:f}) &\leq \widehat{R}(\wh{\:f}) + \frac{ l^2 \gamma^2 \lambda_n^2 k^2 \varepsilon^2 }{(l \gamma (1 - \varepsilon)\lambda_{1} - 1)^4}\\
                &+ \beta + \left(2\beta + \frac{c^2(l+u)}{lu}\right)\sqrt{\frac{\pi(l,u)\ln\frac{1}{\delta}}{2}},
    \end{align*}
    where
    \begin{align*}
        \pi(l,u) &= \frac{lu}{l+u-0.5} \frac{1}{1-1/(2\max\{l,u\})},
    \end{align*}
    and
    \begin{align*}
        \beta &\leq \frac{1.5 k\sqrt{l}}{(l \gamma (1-\varepsilon)\lambda_{1}-1)^{2}} + \frac{\sqrt{2}k}{l \gamma (1 - \varepsilon)\lambda_{1}-1}.
    \end{align*}
\end{theorem}
Theorem~\ref{thm:sparse-ssl-generalization} shows 
how approximating $\Gg$ with $\Hg$ impacts
the generalization error as the number of labeled samples $l$ increases.
If we compare the bound to the exact case
($\varepsilon = 0$), we see that for a fixed $\varepsilon$ the rate of
convergence remains unchanged.
The first term $\varepsilon^2/l^2(1-\varepsilon)^4$ captures the increase
of the empirical error due to the approximation. Since for a fixed
$\varepsilon$ this term scales as $1/l^2$, it is shadowed by the $\beta$ term.
The $\beta$ term itself preserves the same order of convergence,
and is only multiplied by a constant due to the presence of $(1-\varepsilon)$.
In conclusion, for a fixed $\varepsilon$ the approximated algorithm provides
guarantees of the same order as the exact one. This allows us to freely choose
$\varepsilon$ to tradeoff precision and computational complexity.

\begin{proof}

\textbf{Step 1 (generalization of stable algorithms).} When
$\mathcal{L}$ is a transductive algorithm with stability $\beta$, then for 
any $\delta > 0$, with probability at least $1 - \delta$ (w.r.t. the 
randomness of the partition of the graph in labeled and unlabeled sets 
$\trainingset$, $\testset$) the hypotesis $\wt{\:f}$ returned by the algorithm
satisfies
    \begin{align*}
        R(\wt{\:f}) \leq \widehat{R}(\wt{\:f}) + \beta + \left(2\beta + \frac{c^2(l+u)}{lu}\right)\sqrt{\frac{\pi(l,u)\ln\frac{1}{\delta}}{2}},
    \end{align*}
hence it is sufficient to study the stability of sparse-HFS and relate its empirical loss to the result of stable-HFS to obtain the final result.

\textbf{Step 2 (stability).}
Let $\trainingset$ and $\trainingset'$ be two different realizations differing
only in  one label. For simplicity we will assume that $I_{\trainingset}(l,l) = 1$
and $I_{\trainingset}(l+1,l+1) = 0$, and the opposite for $I_{\trainingset'}$.
The original proof \cite{cortes_stability_2008} showed that for
 the two hypotheses returned by stable-HFS,
$\beta \leq \normsmall{\wh{\:f} - \wh{\:f}'}$. Similarly,
for our algorithm $\beta\leq\normsmall{\wt{\:f} - \wt{\:f}'}$.
All that is left is to upper bound  the norm.
The spectral radius of $I_{\trainingset}$ is $1$.
On the other hand, while $\lambda_0 = 0$, the smallest eigenvalue of 
$L_{\Hg}$ restricted to $\funcspace$ is $\lambda_{1}$.
This reduction of the spectral radius of the Laplacian over
the restricted space $\funcspace$ plays a critical role in the proof,
and motivates the choice of this particular constraint.
Let $\:y_\trainingset = I_{\trainingset}\:y$,
$A=P_{\funcspace}(l \gamma L_{\Hg}+I_{\trainingset})$ and $B=P_{\funcspace}(l \gamma L_{\Hg}+I_{\trainingset'})$. The hypotheses $\wt{\:f}$ and $\wt{\:f}'$ returned by sparse-HFS are given by $\wt{\:f}=A^{-1}\:y_{\trainingset}$ and $\wt{\:f}'=B^{-1}\:y_{\trainingset'}$.
We have
\begin{align*}
    \wt{\:f}-\wt{\:f}'&=A^{-1}\:y_{\trainingset}-B^{-1}\:y_{\trainingset'}\\
    &=A^{-1}(\:y_{\trainingset}-\:y_{\trainingset'})+A^{-1}\:y_{\trainingset'}-B^{-1}\:y_{\trainingset'}
\end{align*}
Therefore,
\begin{align*}
    \Vert \wt{\:f}-\wt{\:f}'\Vert \leq\Vert A^{-1}(\:y_{\trainingset}-\:y_{\trainingset'})\Vert+\Vert A^{-1}\:y_{\trainingset'}-B^{-1}\:y_{\trainingset'}\Vert.
\end{align*}
    Noticing that $\funcspace$
    is invariant under $L_{\Hg}$ and that for any vector 
	$P_{\funcspace}$ is an orthogonal 
    projection operator, then by the triangle inequality we immediately 
    obtain that for any $\:f \in \funcspace$
    \begin{align*}
        \Vert P_{\funcspace}(l\gamma L_{\Hg}+I_{S})\:f\Vert&\geq\Vert P_{\funcspace}l\gamma L_{\Hg}\:f\Vert-\Vert P_{\funcspace}I_{S}\:f\Vert\\
        &\geq(l \gamma (1-\varepsilon)\lambda_{1}-1)\Vert \:f\Vert
    \end{align*}
It follows that the spectral radius of the inverse operator $(P_{\funcspace}(l\gamma L_{\Hg}+I_{\trainingset}))^{-1}$ and therefore of $A^{-1}$ and $B^{-1}$ does not exceed $1/(l \gamma (1-\varepsilon) \lambda_{1} - 1)$ when restricted to $\funcspace$ (the inverse is not even defined outside of $\funcspace$).
This together with $\Vert\:y_{\trainingset}-\:y_{\trainingset'}\Vert\leq \sqrt{2}k$
gives us
\begin{align*}
    \Vert A^{-1}(\:y_{\trainingset}-\:y_{\trainingset'})\Vert\leq\frac{\sqrt{2}k}{l \gamma (1 - \varepsilon)\lambda_{1}-1}
\end{align*}
On the other hand, it can be checked that $\Vert\:y_{\trainingset'}\Vert\leq \sqrt{l}k$. Noticing that the spectral radius of $P_{\funcspace}(I_{\trainingset}-I_{\trainingset'})$ cannot exceed $\sqrt{2}<1.5$, we obtain:
\begin{align*}
    &\Vert A^{-1}\:y_{\trainingset'}-B^{-1}\:y_{\trainingset'}\Vert=\Vert B^{-1}(B-A)A^{-1}\:y_{\trainingset'}\Vert\\
    &=\Vert B^{-1}P_{\funcspace}(I_{\trainingset}-I_{\trainingset'})A^{-1}\:y_{\trainingset'}\Vert\leq\frac{1.5 k\sqrt{l}}{(l \gamma (1-\varepsilon)\lambda_{1}-1)^{2}}
\end{align*}
Putting it all together
\begin{align*}
\Vert \wt{\:f}-\wt{\:f}'\Vert \leq \frac{1.5 k\sqrt{l}}{(l \gamma (1-\varepsilon)\lambda_{1}-1)^{2}} + \frac{\sqrt{2}k}{l \gamma (1 - \varepsilon)\lambda_{1}-1}
\end{align*}
\textbf{Step 3 (empirical error).} We can now proceed with the proof of Thm.~\ref{thm:sparse-ssl-generalization}.
    We have already bounded $\beta$ when using $\Hg$ instead of $\Gg$. We can
     also provide guarantees for the difference in the empirical error using
    the sparsifier. Given $\wt{Q} = P_{\funcspace}(l \gamma L_{\Hg}+I_{\trainingset})$, $\wh{Q} = P_{\funcspace}(l \gamma L_{\Gg}+I_{\trainingset})$ we have
    \begin{align*}
        \wh{R}(\wt{\:f}) &= \frac{1}{l}\sum_{i=1}^{l} \left(\wt{\:f}(x_i) - \:y(x_i)\right)^2\nonumber\\
                         &= \tfrac{1}{l} \normsmall{I_\trainingset \wt{\:f} -I_\trainingset \wh{\:f} +I_\trainingset \wh{\:f}  - \:y_{\trainingset}}^2\nonumber\\
                         &\leq \tfrac{1}{l} \normsmall{I_\trainingset \wh{\:f}  - \:y_{\trainingset}}^2 + \frac{1}{l} \normsmall{I_\trainingset \wt{\:f} -I_\trainingset \wh{\:f} }^2 \nonumber\\
        &\leq \wh{R}(\wh{\:f}) + \tfrac{1}{l} \normsmall{I_\trainingset(\wt{Q}^{-1} -\wh{Q}^{-1}) \:y_{\trainingset} }^2\nonumber\\
        &\leq \wh{R}(\wh{\:f}) + \tfrac{1}{l} \normsmall{\wh{Q}^{-1}(\wh{Q} -\wt{Q})\wt{Q}^{-1} \:y_{\trainingset} }^2\nonumber\\
        &\leq \wh{R}(\wh{\:f}) + \frac{lk^2}{l(l \gamma (1 - \varepsilon)\lambda_{1} - 1)^4} \normsmall{\wh{Q} -\wt{Q}}^2
    \end{align*}
    We now need to bound 
    $\normsmall{\wh{Q} -\wt{Q}}^2 = \normsmall{P_{\funcspace} l \gamma (L_{\Gg} - L_{\Hg})}^2$.
    Let $y = L_{\Gg}^{1/2}x$ and $\wt{P}_{\funcspace} = L_{\Gg}^{-1/2}L_{\Hg}L_{\Gg}^{-1/2}$. Def.~\ref{def:eps-sparsifier} implies
\begin{align*}
    &(1-\vareps) P_{\funcspace}\leq \wt{P}_{\funcspace} \leq (1 + \vareps) P_{\funcspace}.
\end{align*}
    Since $\Hg$ is a sparsifier of $\Gg$, by definition we get
    \begin{align*}
        \normsmall{P_{\funcspace} & l \gamma (L_{\Gg} - L_{\Hg})}^2 \leq l^2 \gamma^2 \normsmall{ L_{\Gg} - L_{\Hg}}^2\nonumber\\
        &\leq l^2 \gamma^2 \normsmall{L_{\Gg}^{1/2}(P_{\funcspace} - \wt{P}_{\funcspace})L_{\Gg}^{1/2}}^2 \leq l^2 \gamma^2 \lambda_n^2 \varepsilon^2.
    \end{align*}
    The statement of the theorem is obtained by the combination of the above.
\end{proof}

\section{Experiments}

\begin{figure}
    \centering
    \begin{tabular}{>{\footnotesize\centering\arraybackslash}m{2cm}>{\footnotesize\centering\arraybackslash}m{6cm}}
        \subfigure[]{\includegraphics[height=5cm]{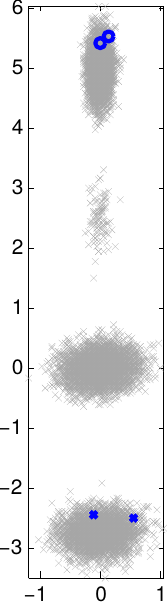} \label{fig:gen-data}}&
        \subfigure[]{\includegraphics[height=5.3cm]{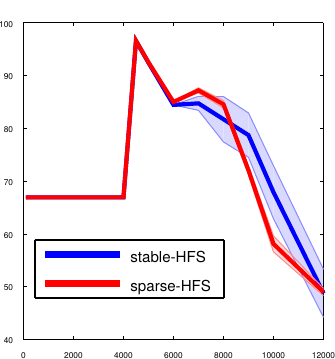}\label{fig:gen-error}}\\
    ~ 
    \vspace{-1cm}
    \begin{itemize}[noitemsep,topsep=0pt,parsep=0pt,partopsep=0pt]
            \item[\subref{fig:gen-data}] {Data}
            \item[\subref{fig:gen-error}] {$R(\wh{\:f})$ vs $R(\wt{\:f})$}
            \item[\subref{fig:edge-ratio}] {$|\Hg|/|\Gg|$}
    \end{itemize}
        &\subfigure[]{\includegraphics[height=2cm]{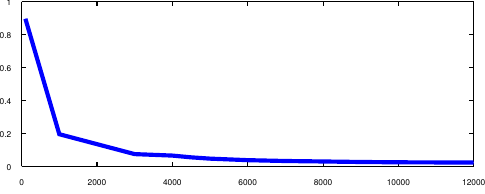}\label{fig:edge-ratio}}
\end{tabular}
\end{figure}
We evaluate the proposed algorithm on the $\Re^2$ data reported in Fig.~\subref{fig:gen-data}
which we designed to show the effect of the spectral sparsification in the case 
when a rather dense graph is needed for a good performance.
The dataset is composed of
$n = 12100$ points, where the two upper clusters belong to one class and the two
lower to the other. We build an unweighted, $k$-nn graph $\Gg$
for
$k = {100,\ldots,12000}$. This gives us values for $m$ ranging from $1.21 \times 10^6$
to $1.38 \times 10^8$ edges. After constructing the graph, we randomly select two points
 from the uppermost and and two from lowermost cluster as our labeled set $\trainingset$.
We then run sparse-HFS to compute $\Hg$ and $\wt{\:f}$, and run stable-HFS
on $\Gg$ to compute $\wh{\:f}$, both with $\gamma = 1$. For sparse-HFS we set
$\varepsilon = 0.8$
Using the labels in $\testset$ we compute the generalization error $R$,
which corresponds to the accuracy.
Fig.~\subref{fig:gen-error} reports the performance of the two algorithms.
Both algorithms fail to recover a good solution until $k > 4000$. This
is due to the fact that until a certain threshold of neighbours is not surpassed, each
cluster remains separated and the labels cannot propagate.
Even after this threshold, sparse-HFS cannot consistently outperform stable-HFS in accuracy. This is
because they are both trying to approximate the stable-HFS solution, but
sparse-HFS uses an approximated matrix~$\Hg$. Nonetheless, the difference in performance
is not large, especially near the optimum.
This is in line with the theoretical analysis that shows that the contribution
due to the approximation error has the same order of magnitude as the
other elements in the bound.
Furthermore, in Fig.~\subref{fig:edge-ratio}
we report the ratio of the number of edges in the sparsifier $\Hg$ over
the number of edges in the orginal graph $\Gg$. Since $\Hg \subseteq \Gg$,
this quantity is always smaller than one, but we can see that for $k = 4500$,
where the accuracy is at its maximum, the sparsifier contains only about 10\%
as many edges as the original graph, with similar accuracy.

\section{Conclusions and Future Work}

We introduced Sparse-HFS, a scalable algorithm that can compute
solutions to SSL problems using only $\bigotime(n\polylog(n))$ space and
$\bigotime(m\polylog(n))$ time. This is achieved in the semi-streaming setting, where
a stream of edges insertion is presented to the algorithm.
Extending this approach to also deal with edge removals in the stream may not be 
trivial.  The approach taken in \cite{kapralov_single_2014}
resorts to sketches to
keep track of all the updates, but this implicit representation requires
$\bigotime(n^2 \polylog(n))$ time to compute the final SSL solution.
In the large scale setting we target an $\bigotime(n^2)$ operation is too
costly to meet our amortized cost, so we limit our attention to insertion-only streaming setting.
Extending sparsification techniques to the full dynamic setting
in a computationally efficient manner is an interesting open problem.

\ifdefined\isaccepted
{\small
\paragraph*{Acknowledgments}
We would like to thank Ioannis Koutis for many useful discussions.
}
\else \newpage
\fi

\bibliography{workshop_sparse_ssl,library}
\bibliographystyle{icml2015}

\end{document}